\newcommand{\aiversion}[1]{}
\newcommand{\fullversion}[1]{#1}

\documentclass[letterpaper]{article}
\usepackage{times}
\usepackage{helvet}
\usepackage{courier}
\usepackage{epsfig}
\usepackage{amsmath}
\usepackage{amsfonts}
\usepackage{amssymb}
\usepackage{amsthm}
\usepackage{booktabs}

\usepackage{natbib}
\newcommand{\shortcite}[1]{\cite{#1}}

\usepackage{fullpage}

 \usepackage{enumitem}
\fullversion{
\usepackage{thmtools}
\declaretheoremstyle[%
  spaceabove=3pt,%
  spacebelow=10pt,%
  headfont=\normalfont\itshape,%
  qed=\qedsymbol%
]{myproof}
\declaretheorem[name={Proof},style=myproof,unnumbered,
]{prf}
\renewenvironment{proof}{\begin{prf}}{\end{prf}}

\declaretheoremstyle[%
  spaceabove=10pt,%
  spacebelow=8pt,%
  bodyfont=\normalfont\itshape,
]{mytheorem}
\declaretheorem[name={Lemma},style=mytheorem]{lemma}
\declaretheorem[name={Theorem},style=mytheorem]{theorem}
\declaretheorem[name={Proposition},style=mytheorem]{proposition}
\declaretheorem[name={Corollary},style=mytheorem]{corollary}
\declaretheorem[name={Example},style=mytheorem]{example}
\declaretheorem[name={Observation},style=mytheorem]{observation}
}

\aiversion{

\newtheorem{theorem}{Theorem}
\newtheorem{proposition}{Proposition}
\newtheorem{observation}{Observation}
\newtheorem{corollary}{Corollary}

\newtheorem{example}{Example}
}

\usepackage{xspace}

\usepackage{tikz}
\usetikzlibrary{arrows}

\usepackage{mathtools}

\renewcommand{\P}{\text{\normalfont P}}
\newcommand{\NP}{\text{\normalfont NP}}
\newcommand{\PSPACE}{\text{\normalfont PSPACE}}

\newcommand{\FPT}{\text{\normalfont FPT}}

\newcommand{\W}[1]{\text{\normalfont W[#1]}}
\newcommand{\paraNP}{\text{\normalfont para-NP}}
\newcommand{\paraPSPACE}{\text{\normalfont para-PSPACE}}
\newcommand{\Card}[1]{|#1|}

\newcommand{\mtext}[1]{\text{\normalfont #1}}

\newcommand{\CNF}{\mtext{\sc CNF}}

\newcommand{\case}[1]{(\Pi{#1},c)}

\newcommand{\Lsf}{\mtext{\sffamily L}}
\newcommand{\Vsf}{\mtext{\sffamily V}}
\newcommand{\Dsf}{\mtext{\sffamily D}}
\newcommand{\Asf}{\mtext{\sffamily A}}

\newcommand{\Vars}[1]{\mtext{vars\ensuremath{(#1)}}}
\newcommand{\Pre}[1]{\mtext{pre}(#1)}
\newcommand{\Post}[1]{\mtext{post}(#1)}
\newcommand{\restr}{\upharpoonright}
\newcommand{\true}{1}
\newcommand{\false}{0}

\newcommand{\planmod}{\mtext{\sc Plan}\-\mtext{\sc Mod}}
\newcommand{\planmodstar}{\planmod\ensuremath{^{\star}}}
\newcommand{\splanmod}[1]{\ensuremath{\SBs#1\SEs}\mtext{-}\planmod}
\newcommand{\rplanmod}{\ensuremath{R}\mtext{-}\planmod}
\newcommand{\rplanmodstar}{\ensuremath{R}\mtext{-}\planmodstar}

\newcommand{\casemod}{\mtext{\sc Case}\-\mtext{\sc Mod}}
\newcommand{\casemodstar}{\casemod\ensuremath{^{\star}}}
\newcommand{\scasemod}[1]{\ensuremath{\SBs#1\SEs}\mtext{-}\casemod}
\newcommand{\rcasemod}{\ensuremath{R}\mtext{-}\casemod}
\newcommand{\rcasemodstar}{\ensuremath{R}\mtext{-}\casemodstar}

\newcommand{\partitionedclique}{\mtext{{\sc Partitioned}-}\allowbreak{}\mtext{{\sc Clique}}}
\newcommand{\hittingset}{\mtext{{\sc Hitting}-}\allowbreak{}\mtext{{\sc Set}}}
\newcommand{\pwsatcirc}{\ensuremath{p}\mtext{\sc -WSat(CIRC)}}
\newcommand{\LCSI}{\mtext{{\sc Longest}-}\allowbreak{}\mtext{{\sc Common}-}\allowbreak{}\mtext{{\sc Subsequence}-}\allowbreak{}\mtext{\sc I}}

\newcommand{\SB}{\{\,}%
\newcommand{\SM}{\;{:}\;}%
\newcommand{\SE}{\,\}}%
\newcommand{\SBs}{\{}%
\newcommand{\SEs}{\}}%

\allowdisplaybreaks

\begin{document}

\title{Parameterized Complexity Results for Plan Reuse}	

\author{
Ronald de Haan$^{1}$\thanks{Supported by the European Research Council (ERC), project COMPLEX REASON, 239962.} \hspace{10pt}
Anna Roubickova$^{2}$ \hspace{10pt}
Stefan Szeider$^{1\hspace{1pt}*}$\\[10pt]
\mbox{}$^1$Institute of Information Systems,
Vienna University of Technology, Vienna, Austria\\
dehaan@kr.tuwien.ac.at\ \ \ stefan@szeider.net\\[10pt]
\mbox{}$^2$Faculty of Computer Science,
Free University of Bozen-Bolzano, Bolzano, Italy\\
anna.roubickova@stud-inf.unibz.it
}

\date{}

%\setlength\titlebox{2.05in} 
%\nocopyright%
\maketitle%
\begin{abstract}
  Planning is a notoriously difficult computational problem of high
  worst-case complexity.  Researchers have been investing significant
  efforts to develop heuristics or restrictions to make planning
  practically feasible. Case-based planning is a~heuristic
  approach where one tries to reuse previous experience when solving
  similar problems in order to avoid some of the planning
  effort. Plan reuse may offer an interesting
  alternative to plan generation in some settings.

  We provide theoretical results that identify
  situations in which plan reuse is provably tractable.
  We perform our analysis in the framework of parameterized complexity,
  which supports a rigorous worst-case complexity analysis that takes
  structural properties of the input into account in terms of
  parameters. A central notion of parameterized complexity is
  fixed-parameter tractability which extends the classical notion of
  polynomial-time tractability by utilizing the effect of structural
  properties of the problem input.

  We draw a detailed map of the parameterized complexity landscape
  of several variants of problems that arise in the context of
  case-based planning. In particular, we consider the problem of
  reusing an existing plan, imposing various restrictions in terms of
  parameters, such as the number of steps that can be added to the
  existing plan to turn it into a solution of the planning instance at
  hand.
\end{abstract}

\section{Introduction}
Planning is one of the central problems of AI with a wide range of
applications from industry to academics (\citealt{ghallab2004automated}).  Planning gives rise
to challenging computational problems. For instance, deciding whether
there exists a plan for a given planning instance is
\PSPACE{}-complete, and the problem remains at least NP-hard
under various restrictions~(\citealt{Bylander94}).  To overcome this high
worst-case complexity, various heuristics, restrictions, and
relaxations of planning problems have been developed that work
surprisingly well in practical settings (\citealt{hoffmann2001ff,helmert2006fast}).  Among the heuristic
approaches is \emph{case-based planning}, which proceeds from the idea
that significant planning efforts may be saved by reusing previous
solutions (\citealt{KamHen92,Veloso:1994:PLA}).  This approach is based on the assumption that
planning tasks tend to recur and that if the tasks are similar, then
so are their solutions. Empirical evidence suggests that this
assumption holds in many settings, and that the case-based approach
works particularly well if the planning tasks require complex
solutions while the modifications required on the known plans are
considerably small.
 
So far the research on the worst-case complexity of case-based
planning did not take into account the essential assumption that similar planning
tasks require similar solutions. Indeed, as shown by
\shortcite{Liberatore05}, if none of the known solutions are helpful, then
the case-based system needs to invest an effort comparable to
generating the solution from scratch, and hence does not benefit from
the previous experience. There is no way to benefit from the
knowledge of an unrelated solution.  However, the result disregards
the case-based assumptions which are meant to avoid such worst
cases. It seems that the classical complexity framework 
is not well-suited for taking such assumptions into account. 

\paragraph{New Contribution}   
In this paper we provide theoretical results that identify situations in
which the plan reuse of the case-based approach is provably tractable.  We
perform our analysis in the framework of \emph{parameterized
  complexity}, which supports a rigorous worst-case complexity analysis
that takes structural properties of the input into account
(\citealt{DowneyFellowsStege99,Niedermeier06,GottlobSzeider08}).  These
structural properties are captured in terms of \emph{parameters},
which are integer numbers that are small compared to the size of the
total problem input. The theoretical analysis now considers the impact
of these parameters on the worst-case complexity of the considered
problems.  A central notion of parameterized complexity is
\emph{fixed-parameter tractability}, which extends the classical notion
of polynomial-time tractability by utilising the impact of
parameters. Parameterized complexity also provides a hardness theory
that, similar to the theory of NP-completeness, provides strong
evidence that certain parameterized problems are not fixed-parameter
tractable (fixed-parameter \textit{in}tractable).

In the problems we study we are given a planning task together with a
stored solution for a different planning task,
where this solution consists of a plan and an initial state the plan is
applied to. The question is to
modify the existing solution to obtain a solution for the new planning
task. By means of various parameters we control the modifications applied to
the stored solution.  For instance, we can require that the number of
additional planning steps added to fit the stored solution to the new
planning task is small compared to the length of the stored solution.

In order to evaluate the impact of structural properties  on the
overall complexity we use parameters based on the following four
restrictions, each restriction is associated with one of the four
symbols  $\Lsf$, $\Asf$, $\Vsf$, and $\Dsf$. 

\setlist[itemize]{leftmargin=15pt}
\begin{itemize}[itemsep=0pt]
\item[$\Lsf$:] bounds on the number of added planning steps
\item[$\Asf$:] bounds on the size of a specified set of actions from
  which the added planning steps are built (each action from the set
  can be added several times)
\item[$\Vsf$:] bounds on the size of a specified set of variables that
  may be mentioned by the added planning steps
\item[$\Dsf$:] bounds on the size of a specified set of values that
  may be mentioned by the added planning steps
\end{itemize}

\begin{figure}[h!]
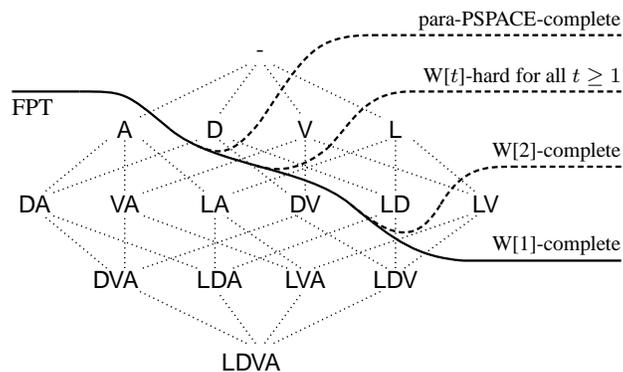

\centering
  
  \vspace{20pt}
  \begin{pgfpicture}{-1cm}{-3cm}{2.5cm}{2cm}
    % Set scale
    \newdimen\pubsdim
    \pgfextractx{\pubsdim}{\pgfpoint{5mm}{0mm}}
    \pgfsetxvec{\pgfpoint{6mm}{0mm}}
    \pgfsetyvec{\pgfpoint{0mm}{5mm}}

    \small
    \pgfsetlinewidth{0.2mm}
    \pgfsetdash{{0.15mm}{0.5mm}}{0mm}

    % Top element
    \pgfputat{\pgfxy(0,4)}{\pgfbox[center,center]{-}}

    % Lines from top
    \pgfline{\pgfxy(-0.3,3.7)}{\pgfxy(-2.7,2.3)} % Top-A
    \pgfline{\pgfxy(-0.1,3.7)}{\pgfxy(-0.9,2.4)} % Top-D
    \pgfline{\pgfxy(0.1,3.7)}{\pgfxy(0.9,2.4)}   % Top-V
    \pgfline{\pgfxy(0.3,3.7)}{\pgfxy(2.7,2.3)}   % Top-L

    % Level 1
    \pgfputat{\pgfxy(-3,2)}{\pgfbox[center,center]{\textsf{A}}}
    \pgfputat{\pgfxy(-1,2)}{\pgfbox[center,center]{\textsf{D}}}
    \pgfputat{\pgfxy(1,2)}{\pgfbox[center,center]{\textsf{V}}}
    \pgfputat{\pgfxy(3,2)}{\pgfbox[center,center]{\textsf{L}}}

    % Lines from A
    \pgfline{\pgfxy(-3.4,1.7)}{\pgfxy(-4.8,0.3)} % A-DA
    \pgfline{\pgfxy(-3.0,1.7)}{\pgfxy(-3.0,0.3)} % A-VA
    \pgfline{\pgfxy(-2.6,1.7)}{\pgfxy(-1.2,0.3)} % A-LA

    % Lines from D
    \pgfline{\pgfxy(-1.4,1.7)}{\pgfxy(-4.6,0.3)} % D-DA
    \pgfline{\pgfxy(-0.9,1.7)}{\pgfxy(0.6,0.3)}  % D-DV
    \pgfline{\pgfxy(-0.6,1.7)}{\pgfxy(2.6,0.3)}  % D-LD

    % Lines from V
    \pgfline{\pgfxy(0.8,1.7)}{\pgfxy(-2.6,0.3)}  % V-VA
    \pgfline{\pgfxy(1.0,1.7)}{\pgfxy(1.0,0.3)}   % V-DV
    \pgfline{\pgfxy(1.2,1.7)}{\pgfxy(4.6,0.3)}   % V-LV

    % Lines from L
    \pgfline{\pgfxy(2.8,1.7)}{\pgfxy(-0.6,0.3)}  % L-LA
    \pgfline{\pgfxy(3.0,1.7)}{\pgfxy(3.0,0.3)}   % L-LD
    \pgfline{\pgfxy(3.4,1.7)}{\pgfxy(4.8,0.3)}   % L-LV

    % Level 2
    \pgfputat{\pgfxy(-5,0)}{\pgfbox[center,center]{\textsf{DA}}}
    \pgfputat{\pgfxy(-3,0)}{\pgfbox[center,center]{\textsf{VA}}}
    \pgfputat{\pgfxy(-1,0)}{\pgfbox[center,center]{\textsf{LA}}}
    \pgfputat{\pgfxy(1,0)}{\pgfbox[center,center]{\textsf{DV}}}
    \pgfputat{\pgfxy(3,0)}{\pgfbox[center,center]{\textsf{LD}}}
    \pgfputat{\pgfxy(5,0)}{\pgfbox[center,center]{\textsf{LV}}}
 
    % Lines from DVA
    \pgfline{\pgfxy(-3.4,-1.7)}{\pgfxy(-4.8,-0.3)} % DVA-DA
    \pgfline{\pgfxy(-3.0,-1.7)}{\pgfxy(-3.0,-0.3)} % DVA-VA
    \pgfline{\pgfxy(-2.6,-1.7)}{\pgfxy(0.8,-0.3)} % DVA-DV

    % Lines from LDA
    \pgfline{\pgfxy(-1.4,-1.7)}{\pgfxy(-4.6,-0.3)} % LDA-DA
    \pgfline{\pgfxy(-1.0,-1.7)}{\pgfxy(-1.0,-0.3)}  % LDA-LA
    \pgfline{\pgfxy(-0.6,-1.7)}{\pgfxy(2.6,-0.3)}  % LDA-LD

    % Lines from LVA
    \pgfline{\pgfxy(0.6,-1.7)}{\pgfxy(-2.6,-0.3)}  % LVA-VA
    \pgfline{\pgfxy(0.9,-1.7)}{\pgfxy(-0.8,-0.3)}  % LVA-LA
    \pgfline{\pgfxy(1.2,-1.7)}{\pgfxy(4.6,-0.3)}   % LVA-LV

    % Lines from LDV
    \pgfline{\pgfxy(2.8,-1.7)}{\pgfxy(0.8,-0.3)}  % LDV-DV
    \pgfline{\pgfxy(3.0,-1.7)}{\pgfxy(3.0,-0.3)}  % LDV-LD
    \pgfline{\pgfxy(3.2,-1.7)}{\pgfxy(4.8,-0.3)}  % LDV-LV

    % Level 3
    \pgfputat{\pgfxy(-3.2,-2)}{\pgfbox[center,center]{\textsf{DVA}}}
    \pgfputat{\pgfxy(-0.9,-2)}{\pgfbox[center,center]{\textsf{LDA}}}
    \pgfputat{\pgfxy(1,-2)}{\pgfbox[center,center]{\textsf{LVA}}}
    \pgfputat{\pgfxy(3,-2)}{\pgfbox[center,center]{\textsf{LDV}}}

    % Lines from LDVA
    \pgfline{\pgfxy(-0.3,-3.7)}{\pgfxy(-2.7,-2.3)} % Bottom-A
    \pgfline{\pgfxy(-0.1,-3.7)}{\pgfxy(-0.9,-2.4)} % Bottom-D
    \pgfline{\pgfxy(0.1,-3.7)}{\pgfxy(0.9,-2.4)}   % Bottom-V
    \pgfline{\pgfxy(0.3,-3.7)}{\pgfxy(2.7,-2.3)}   % Bottom-L

    % Bottom
    \pgfputat{\pgfxy(-0.2,-4.2)}{\pgfbox[center,center]{\textsf{LDVA}}}

    %% Old standard complexity results %%

    %% Solid line style
    \pgfsetlinewidth{0.3mm}
    \pgfsetdash{}{0mm}

    % Separation for FPT
    \pgfxyline(-5.5,3.0)(-4.0,3.0)
    \pgfxycurve(-4.0,3.0)(-3.0,3.0)(-3.0,3.0)(-2.0,2.0)
    \pgfxycurve(-2.0,2.0)(-1.0,1.0)(1.0,1.0)(2.0,0.0)
    \pgfxycurve(2.0,0.0)(3.0,-1.0)(3.5,-1.4)(4.5,-1.5)
    \pgfxyline(4.5,-1.5)(8.0,-1.5)
    \pgfstroke
    %\pgfputat{\pgfxy(-6.5,3.2)}{\pgfbox[left,bottom]{\paraPSPACE{}-c}}
    \pgfputat{\pgfxy(8.0,4.65)}{\pgfbox[right,bottom]{\footnotesize \paraPSPACE{}-complete}}
    \pgfputat{\pgfxy(-5.5,2.8)}{\pgfbox[left,top]{FPT}}
    \pgfputat{\pgfxy(8.0,3.15)}{\pgfbox[right,bottom]{\footnotesize \W{\ensuremath{t}}-hard for all $t \geq 1$}}
    \pgfputat{\pgfxy(8.0,1.15)}{\pgfbox[right,bottom]{\footnotesize \W{2}-complete}}
    \pgfputat{\pgfxy(8.0,-1.35)}{\pgfbox[right,bottom]{\footnotesize \W{1}-complete}}
    
    %% Dashed line style
    \pgfsetlinewidth{0.3mm}
    \pgfsetdash{{0.77mm}{0.48mm}}{0mm}    
    
    % Separation for para-PSPACE
%    \pgfxycurve(-2.0,2.0)(-1.1,1.3)(-0.6,1.3)(0.0,2.2)
    \pgfxycurve(-1.4,1.56)(-0.9,1.3)(-0.6,1.3)(0.0,2.2)
    \pgfxycurve(0.0,2.2)(1.0,4.0)(1.2,4.5)(2.5,4.5)
    \pgfxyline(2.49,4.5)(8.0,4.5)

    % Separation for W[t]-h
%    \pgfxycurve(-0.9,1.33)(-0.6,1.2)(1.0,1.0)(1.57,1.63)
%    \pgfxycurve(1.57,1.63)(2.0,2.0)(2.3,3.0)(3.5,3.0)
%    \pgfxyline(3.5,3.0)(6.5,3.0)

    \pgfxycurve(-0.1,1.03)(0.5,0.8)(1.0,1.0)(1.5,1.5)
    \pgfxycurve(1.5,1.5)(2.0,2.1)(2.3,3.0)(3.32,3.0)
    \pgfxyline(3.31,3.0)(8.0,3.0)

    % Separation for W[2]
    \pgfxycurve(2.0,0.0)(3.0,-1.0)(3.5,-1.0)(4.0,0.0)
    \pgfxycurve(4.0,0.0)(4.5,1.0)(5.0,1.0)(5.59,1.0)
    \pgfxyline(5.58,1.0)(8.0,1.0)

    % Separation for FPT and W[2]-h, given post-uniqueness
    %\begin{pgftranslate}{\pgfxy(0.0,0.1)}
    %  \pgfxyline(-6.5,3.0)(-4.0,3.0)
    %  \pgfxycurve(-4.0,3.0)(-3.0,3.0)(-3.0,3.0)(-2.0,2.1)
%      \pgfxyline(-2.0,2.0)(2.5,-2.5)
    %  \pgfxycurve(-2.0,2.1)(-1.0,0.8)(0.0,2.0)(0.0,2.0)
    %  \pgfxycurve(0.0,2.0)(2.0,2.0)(2.0,3.0)(4.5,3.0)
    %  \pgfxyline(4.5,3.0)(6.5,3.0)
    %  \pgfstroke
    %  \pgfputat{\pgfxy(6.5,2.8)}{\pgfbox[right,top]{fp-tractable}}
    %\end{pgftranslate}
  \end{pgfpicture}
  \vspace{-10pt}
  \caption{\small{Overview of the fixed-parameter (in)tractability results
  for all combinations of restrictions $\Lsf,\Asf,\Vsf,\Dsf$.}}
  \label{fig:lattice}
\end{figure}

\noindent We show that parameterized by $\Lsf$ the problem is fixed-parameter
intractable even if the additional steps may only be added to the
beginning and end of the plan.
%However, for post-unique actions the
%problem becomes fixed-parameter tractable. 
Parameterized by $\Asf$
the problem is fixed-parameter tractable.  Parameterized either by
$\Vsf$ or $\Dsf$, the problem is fixed-parameter intractable; however,
if we combine these two parameters, we achieve fixed-parameter
tractability.
Combining the restriction $\Lsf$ with either $\Vsf$ or $\Dsf$
is not enough to achieve fixed-parameter tractability.

%The parameterized complexities of other combinations of these
%parameters follow from these essential results, and
We obtain a full
classification as shown in Figure~\ref{fig:lattice}. In addition, we
show that the same results hold even if we reuse only some ``infix''
of the stored solution, i.e., if it is allowed to discard any number
of actions from the beginning and end of the stored solution before
the modification takes place.  Finally, we prove that in more
general settings where we reuse only the syntactical sequence of
actions as represented by the stored solution (disregarding the actual
states of the stored solution),  all the combinations of parameters
considered yield fixed-parameter intractability.

\section{Preliminaries}
In this section we introduce basic notions and notation related to (case-based) planning and parameterized complexity, which are used further in the paper.

\paragraph{Planning}
In this study, we use the SAS+ planning framework (\citealt{BackstromNebel96}) in the notational variant of \cite{ChenGimenez10}: An instance of the planning problem, or \emph{a planning instance}, is a tuple $\Pi = (V,I,G,A)$,
whose components are described as follows.

\begin{itemize}[itemsep=0pt,leftmargin=*]
\item $V$ is a finite set of variables, where each $v \in V$ has an
  associated finite domain $D(v)$.
  % Note that variables are not necessarily propositional, i.e.,
  % $D(v)$ may have any finite size.
  A \emph{state} $s$ is a mapping defined on a set $V$ of  variables  such
  that $s(v) \in D(v)$ for each $v \in V$.  A \emph{partial state} $p$
  is a mapping defined on a subset $\Vars{p}$ of $V$
  such that for all $v \in \Vars{p}$ it holds that $p(v) \in D(v)$. We
  sometimes denote a partial state $p$ by a set of explicit mappings
  $\SB v \mapsto p(v) \SM v \in \Vars{p} \SE$.
\item $I$ is a state called the \emph{initial state}.
\item $G$ is a partial state called the \emph{goal}.
\item $A$ is a set of actions; each action $a \in A$ is of the form $a =
  (\Pre{a} \Rightarrow \Post{a})$, where $\Pre{a}$ is a partial state called
  \emph{precondition}, and $\Post{a}$ is a partial state called
  \emph{postcondition}.
\end{itemize}
  For a (partial) state $s$ and a subset $W \subseteq
    V$, we let $(s \restr W)$ be the (partial) state resulting from
  restricting $s$ to $W$. We say that a (partial) state $s$ is a
  \emph{goal state}, or that $s$ satisfies the goal, if $(s \restr \Vars{G}) = G$.  A \emph{plan}
  (for an instance $\Pi$) is a sequence of actions $p =
  (a_1,\dotsc,a_n)$.
  The \emph{application} of a plan $p$ on a state $s$ yields a state $s[p]$, which is defined inductively as follows.
  The application of an empty plan ($p = \epsilon$) does not change the state ($s[\epsilon] = s$).
  For a non-empty plan $p = (a_1,\dotsc,a_n)$,
  we define $s[p]$ based on the inductively defined state $s[p']$, where $p' = (a_1,\dotsc,a_{n-1})$.
\begin{itemize}[itemsep=0pt,leftmargin=*]
  \item If $(s[p'] \restr \Vars{\Pre{a_n}}) \neq \Pre{a_n}$ then $s[p] = s[p']$,
  i.e., if the precondition of $a_n$ does not hold in $s[p']$, the action $a_n$ is not applicable and does not change the state.
  \item Otherwise, $a_n$ is applicable and $s[p]$ is the state equal to $\Post{a_n}$ on variables $v \in \Vars{\Post{a_n}}$, and equal to $s[p']$ on $v \in V \setminus \Vars{\Post{a_n}}$.
\end{itemize}
A plan $p$ is a \emph{solution plan} if $I[p]$ is a goal state. 

If $\Card{D(v)} \leq 2$ for each $v \in V$, then we have a
\textit{Boolean} (or \textit{binary}) instance, which in fact gives us a
notational variant of the STRIPS planning framework (\citealt{Bylander94}).
%Another interesting syntactical restriction on sets of actions,
%called \emph{post-uniqueness}, has proved to be useful for
%identifying restricted cases of planning that are of lower
%(parameterized) complexity \cite{BackstromNebel96,BackstroemJonssonChenOrdyniakSzeider12}. We say that a
%set $A' \subseteq A$ of actions is post-unique if for each $v \in V$ and for each $d \in D(v)$, there
%is at most one $a \in A'$ such that $\Post{a}$ is defined on $v$ and
%$\Post{a}(v) = d.$

Consider the following example instance,
that we will use as a running example in the
remainder of this paper.
\begin{example}
\label{ex:pi}
We let $\Pi = (V,I,G,A)$ be the planning instance defined below.
A solution plan for $\Pi$ would be $p~=~(a_3,a_1,a_2)$.
\[ \begin{array}{r l}
  V & = \SBs v_1,v_2,v_3 \SEs \\[5pt]
  & \quad D(v_1) = \SBs 0,1,2 \SEs \\
  & \quad D(v_2) = D(v_3) = \SBs 0,1 \SEs \\[5pt]
  I & = \SBs v_1 \mapsto 0, v_2 \mapsto 0, v_3 \mapsto 0 \SEs \\
  G & = \SBs v_1 \mapsto 2 \SEs \\
  A & = \SBs a_1, a_2, a_3, a_4 \SEs \\[5pt]
  & \quad a_1 = (\SBs v_1 \mapsto 0, v_2 \mapsto 1 \SEs, \SBs v_1 \mapsto 1 \SEs) \\
  & \quad a_2 = (\SBs v_1 \mapsto 1, v_2 \mapsto 1 \SEs, \SBs v_1 \mapsto 2 \SEs) \\
  & \quad a_3 = (\emptyset, \SBs v_2 \mapsto 1, v_3 \mapsto 1 \SEs) \\
  & \quad a_4 = (\emptyset, \SBs v_3 \mapsto 0 \SEs) \\
\end{array} \]
\end{example}

\paragraph{Case-Based Planning}
\sloppypar
Case-based planning (CBP) is a type of case-based reasoning that involves the use of stored experiences (called \emph{cases}) of solving analogous problems. %In order to benefit from remembering and reusing past plans, a CBP system needs efficient methods for retrieving analogous cases and for adapting retrieved plans together with a library of sufficient size and coverage to yield useful analogues. 
Often, a case is composed of a planning instance $\Pi'~=~(V,J,H,A)$ and a solution plan $c$ of $\Pi'$. The plan $c$ can be replaced by some other information related to the search for a solution to $\Pi'$, e.g., a set of justifications~(\citealt{KamHen92,Veloso:1994:PLA,HanWel95}). A plan library, or a \emph{case base}, is a collection of such cases, constituting the experience of the planner. For more detailed explanation of implementation choices of specific planners we refer to the survey of Spalazzi~(\citeyear{Spalazzi2001}).
\begin{example}\label{ex:case}
Consider the case $\case{'}$, where the planning instance $\Pi'$
coincides with the instance defined in Example~\ref{ex:pi} on $V,A$,
and where $J = \SBs v_1 \mapsto 0, v_2 \mapsto 1, v_3 \mapsto 0 \SEs$
and $H = \SBs v_1 \mapsto 2 \SEs$.
The solution plan $c = (a_1,a_2)$ can be reused to find the solution $p$
given in Example~\ref{ex:pi}.
\end{example}
\noindent When faced with a new problem, the case-based planner follows a sequence of steps common in case-based reasoning~(\citealt{Aamodt94}). First, it queries the library to \emph{retrieve} cases suitable
%\footnote{A case suitable to adaptation has an adaptation cost that is lower with respect to the other candidates of the case base and with respect to plan generation~\cite{DBLP:journals/ai/Serina10}.} 
for reuse. The \emph{reuse} step modifies the retrieved solution(s) to solve the new problem and such a new solution is validated in the \emph{revision} phase by execution, simulated execution, etc. The verified solution may be eventually stored in the case base during a \emph{retention} process.

%It is highly unlikely that instances from different planning domains
%would have (structurally) similar solutions up to the level which
%allows a successful reuse. Therefore, a library of plans introduces a domain dependent knowledge to the system. However, case-based planners usually use no domain specific algorithms, so they can be applied to any planning problem (provided an apropriate case base) and in this sense CBP constitutes a domain independent approach. 

%Even though the library is usually linked to a specific domain (and in some cases even to the size of the problem), the system should be capable of renaming the objects of the problem in order to increase the reusability of the stored solutions; e.g., in logistics domain the names of the packages can be permuted, whereas the structure of the problem remains very similar, if not even the same. Therefore many case-based planners perform an \emph{object matching} step~\cite{Nebel95,DBLP:journals/ai/Serina10} at the beginning of the adaptation process in order to rename the problem's objects to increase the similarity of the problem to the cases stored in the library. The object matching is an {\NP}-hard problem and to address it, case-based planners use different heuristics. 

In this paper, we focus on theoretical properties of plan reuse and therefore we skip the
retrieval, we simply assume that together with the problem to solve we are also given a case which contains a suitable solution for reuse.
In other words, we assume that the instance $\Pi$ to be solved coincides on the sets $V, A$ with the instance $\Pi'$ provided in the case.%
\paragraph{Parameterized Complexity}
Here we introduce the relevant concepts of parameterized complexity theory.
For more details, we refer to the works of \shortcite{DowneyFellows99},
\shortcite{DowneyFellowsStege99},
\shortcite{FlumGrohe06}, \shortcite{Niedermeier06},
and \shortcite{GottlobSzeider08}.

In the traditional setting of considering the complexity of a problem,
the input size $n$ of the instance is the only measure available.
Parameterized complexity is a two-dimensional framework to
classify the complexity of problems based on their input size $n$
and some additional parameter $k$.
An instance of a parameterized problem is a pair $(I,k)$
where $I$ is the main part of the instance,
and $k$ is the \emph{parameter}.
A parameterized problem is \emph{fixed-parameter tractable}
if it can be solved by a fixed-parameter algorithm, i.e.,
if instances $(I,k)$ can be solved in time $O(f(k)n^c)$,
where $f$ is a computable function of $k$,
$c$ is a constant, and $n$ is the size of $I$.
\FPT{} denotes the class of all fixed-parameter tractable decision
problems.
Many problems that are classified as
intractable in the classical setting can be shown to be
fixed-parameter tractable.

Parameterized complexity also offers a \emph{completeness theory},
similar to the theory of \NP{}-completeness.
This allows the accumulation of strong theoretical evidence that
a parameterized problem is not fixed-parameter tractable.
Hardness for parameterized complexity classes is based on \emph{fpt-reductions},
which are many-one reductions where the parameter of one problem
maps into the parameter for the other.
A parameterized problem $L$ is fpt-reducible to another
parameterized problem $L'$ if there is a mapping $R$
from instances of $L$ to instances of $L'$ such that
(i) $(I,k) \in L$ if and only if $(I',k') = R(I,k) \in L'$,
(ii) $k' \leq g(k)$ for a computable function $g$, and
(iii) $R$ can be computed in time $O(f(k)n^c)$
for a computable function $f$ and a constant $c$,
where $n$ is the size of $I$.

Central to the completeness theory is the hierarchy of parameterized complexity classes
$\FPT{} \subseteq \W{1} \subseteq \W{2} \subseteq \dotsm \subseteq \W{P} \subseteq \paraPSPACE{}$,
where all inclusions are believed to be strict.
Each of the classes \W{\ensuremath{t}} for $t \geq 1$ and \W{P}
contains all parameterized problems that can be reduced
to a certain parameterized satisfiability problem
under fpt-reductions.
For instance, for $\W{2}$, the corresponding satisfiability problem
asks whether a given \CNF{} formula has a satisfying assignment
that sets exactly $k$ variables to true.
A sufficient condition for a problem to be hard for the class \paraPSPACE{}
is that the problem is \PSPACE{}-hard for a single value of the parameter
(\citealt{FlumGrohe03}).
There is strong evidence that a parameterized problem
that is hard for any of these intractability classes is not in \FPT{}.
%Fixed-parameter tractability of any problem hard for these classes
%would imply that the Exponential Time Hypothesis fails \cite{FlumGrohe06,ImpagliazzoPaturiZane01}
%(i.e., the existence of a $2^{o(n)}$ algorithm for $n$-variable $3\SAT{}$).

We use the following problems to prove some fixed-parameter intractability results.

\begin{quote}
\partitionedclique{} is a \W{1}-complete problem
(\citealt{FellowsHermelinRosamondVialette09}).
The instances are tuples $(V,E,k)$,
where $V$ is a finite set of vertices
partitioned into $k$ subsets $V_1,\dotsc,V_k$,
$(V,E)$ is a simple graph,
and $1 \leq k$ is a parameter.
The question is whether there exists a $k$-clique
in $(V,E)$ that contains a vertex in each $V_i$.
\end{quote}

\begin{quote}
\hittingset{} is a \W{2}-complete problem
(\citealt{DowneyFellows95}).
The instances are tuples $(S,C,k)$,
where $S$ is a finite set of nodes,
$C$ is a collection of subsets of $S$,
and $1 \leq k \leq \Card{C}$ is a parameter.
The question is whether
there exists a hitting set $H \subseteq S$
such that $\Card{H} \leq k$
and $H \cap c \neq \emptyset$ for all $c \in C$.
\end{quote}

\begin{quote}
\pwsatcirc{} (weighted circuit satisfiability) is a \W{P}-complete problem
(\citealt{DowneyFellows95}).
The instances are pairs $(C,k)$,
where $C$ is a Boolean circuit,
and $1 \leq k$ is a parameter.
The question is whether there exists a satisfying assignment
of $C$ that sets at most $k$ input nodes to true.
\end{quote}

\begin{quote}
\LCSI{} is a parameterized problem that is \W{\ensuremath{t}}-hard for all $t \geq 1$
(\citealt{BodlaenderDowneyFellowsWareham95}).
As input, it takes $k$ strings $X_1,\dotsc,X_k$ over
an alphabet $\Sigma$, and a positive integer $m$.
The parameter is $k$.
The question is whether there is a string $X \in \Sigma^{*}$
of length at least $m$ that is a subsequence
of $X_i$ for all $1 \leq i \leq k$.
\end{quote}

\section{Related Work}
The first paper providing a complexity-theoretical study of plan reuse~(\citealt{Nebel95}) considered so-called \emph{conservative} plan reuse.
% in detail and one aspect of retrieval, namely the object matching.
Conservative plan reuse maximizes the unchanged part of the known solution.
%They defined a number of  decision problems capturing different modification strategies (deleting/inserting actions into the plan, modifying the actions at the beginning and/or anywhere in the plan) and showed that in no settings the plan reuse is provably more efficient than plan generation.
The authors showed that such a plan reuse is not provably more efficient than plan generation.
Moreover, they show that
%the conservativeness is in fact an additional source of hardness as the problem of 
identifying what is the maximal reusable part of the stored solution
% is \NP-hard
is an additional source of hardness.
%Further, the problem of finding the optimal renaming of variables and values to match the new instance as closely as possible to the instance stored in the case is \NP-hard in the STRIPS formalisaton.

Liberatore~(\citeyear{Liberatore05}) studied the
%somewhat complementary problem,
problem of plan reuse in a different fashion,
interpreting the case (or the case base) as a ``hint" that makes the search for the solution plan more informed.
%Not surprisingly, by a sequence of theorems he shows that neither a case, nor a case base, and in fact no polynomially big data structure can provide a ``hint" that would lead to provably improved complexity of \textsc{PlanSat} (a decision problem whether a given planning instance has a solution).
The complexity results he provides do not improve over the complexity of uninformed plan generation.
%However, he shows that \textsc{PlanSAT} is compilable to \P{} if the new planning instance $\Pi$ differs from the stored instance $\Pi'$ only in a constant number of \vocab{value-assignments/valuations} in the initial state and goal.
He does however give a tractable compilation result
for planning instances that differ from the stored instances
only in a constant number of valuations from the initial state and goal.

The parameterized complexity of planning was first studied by
%\citeauthor{DowneyFellowsStege99}
\shortcite{DowneyFellowsStege99}
and, more recently, by
%\citeauthor{BackstroemJonssonChenOrdyniakSzeider12}
\shortcite{BackstroemJonssonChenOrdyniakSzeider12,BackstroemJonssonOrdyniakSzeider13}, using the solution
length as the parameter.  The analysis by B\"{a}ckstr\"{o}m et al.
reveals that the planning problem is \W{2}-complete and there exist
fragments that are \W{1}-complete and other fragments that are
fixed-parameter tractable. More specifically, they provide a full
classification of SAS+ planning under all combinations of the P, U, B
and S restrictions introduced by %\citeauthor{BackstromKlein91}
\shortcite{BackstromKlein91}, and they provide a full
classification of STRIPS planning under the syntactical restrictions
studied by \shortcite{Bylander94}.

\section{Parameterized Complexity of Plan Reuse}
In this paper, we study the parameterized complexity of reusing a plan. However, as this work is motivated by plan reuse in the context of case-based planning, we exploit assumptions common in the case-based approaches
to ensure that there is a solution at hand that can be reused.
Also, we consider a more specific form of a plan reuse (\emph{case reuse}) and generalizations thereof.

\paragraph{Reusing the Case}
\noindent In its general form, the classical complexity of plan reuse is not better than the one of plan generation. \shortcite{Liberatore05} has shown it to be \PSPACE{}-complete. However, the case-based approach assumes that similar problems have similar solutions (\citealt{Leake96}). This means that cases can either be used to yield a solution by applying only a limited number of modifications or will not be helpful at all in finding a solution. When considering the complexity of plan reuse in the classical setting,
we cannot exclude the worst case in which the case provides no guidance
and where an uninformed search similar to the traditional plan generation is needed.
When using the framework of parameterized complexity instead,
we can capture the computational complexity of reusing a case in those settings
where it can be used to get a solution plan with only a limited amount of modification.

%We will restrict ourselves to the setting of (grounded) planning with finite variable domains. Also, we assume that the case is given as part of the input.

We consider a case $\case{'}$ useful for solving an instance $\Pi$ if $c$
can be modified to a solution plan $p$ for $\Pi$ by means of limited modification.
We define the following template $\casemod$ for the decision problems,
intended to find such useful cases.
%\pagebreak
\begin{quote}
  $\casemod$
  
  \emph{Instance:} a planning instance $\Pi = (V,I,G,A)$;
  a case $\case{'}$ consisting of an instance $\Pi'~=~(V,J,H,A)$\footnote{In the remainder of the paper, we will often specify an instance $\Pi'$ in the definition above only by its value of $J$. Since $\Pi$ and $\Pi'$ coincide on $V$ and $A$, and the choice of $H$ is not relevant for answering the question, this will suffice for most purposes.} and its solution plan $c = (c_1,\dotsc,c_l)$; a subset of actions $A' \subseteq A$; and an integer $M$.

  \emph{Question:} Does there exist a sequence of actions $(g_1,\dotsc,g_m) \in (A')^m$
  for some $m \leq M$, and does there exist some $0 \leq i \leq m$,
  such that $(g_1,\dotsc,g_i,c_1,\dotsc,c_l,g_{i+1},\dotsc,g_m)$
  is a solution plan for $\Pi$ and $I[(g_1,\dotsc,g_i)] = J$?
\end{quote}
The sequence $g = (g_1,\dotsc,g_m)$ in the definition above can be thought of as the ``glue'' that enables the reuse of the plan $c$ by connecting the new initial state $I$ to the beginning of the case $\case{'}$, using the plan $c$ to reach its goal and connecting it to the goal required by instance $\Pi$. In the following, we will often refer to these action occurrences (or steps) as \emph{glue steps}. Though such a reuse may seem naive, $\casemod$ is in fact implemented and used by CBP system \textsc{FarOff} (\citealt{TonidandelR02}). 

\begin{example}
\label{ex:instance}
Let $\Pi$ be the planning instance from Example~\ref{ex:pi} and $\case{'}$ the case from Example~\ref{ex:case}.
Consider the instance for $\casemod{}$,
given by $(\Pi,(\Pi',c),A',M)$,
where $A' = \SBs a_3,a_4 \SEs$ and $M=3$.
This is a positive instance,
since the solution plan $p = (a_3,a_4,a_1,a_2)$
can be constructed from $c$ by adding
the sequence of actions
$(a_3,a_4)$ from $A'$,
$I[(a_3,a_4)] = J$ and $\Card{(a_3,a_4)} \leq M$.
\end{example}

We will consider a number of different parameterizations for $\casemod$,
where in each case the parameter is intended to capture
the assumption that the plan given in the case is similar to the solution we are looking for.
In order to define these variants, we define the problems $\rcasemod{}$, for any subset $R$ of $\SBs \Lsf, \Vsf, \Dsf, \Asf \SEs$. The choice of the parameterization depends on this set $R$ of restrictions.
\begin{itemize}[itemsep=0pt,leftmargin=*]
  \item If $R$ includes $\Lsf$, we add to the parameterization the allowed maximum length of the glue sequence $g$.
  \item If $R$ includes $\Vsf$, we add to the parameterization the number of variables mentioned in the actions in $A'$.
  \item If $R$ includes $\Dsf$, we add to the parameterization the number of values mentioned in the actions in $A'$.
  \item If $R$ includes $\Asf$, we add to the parameterization the number of actions in $A'$.
\end{itemize}
For instance, the parameter in the problem $\scasemod{\Lsf,\Asf}$ is $k + l$, where $k$ is the maximum length allowed for the sequence of glue steps and $l = \Card{A'}$.

In order to establish the parameterized complexity landscape
for various combinations of these restrictions
as sketched in Figure~\ref{fig:lattice},
we need to prove the following results.
We show that $\rcasemod{}$ is fixed-parameter tractable
for $R \in \SBs \SBs A \SEs, \SBs V,D \SEs \SEs$,
that it is \W{1}-complete for $R = \SBs L,V \SEs$,
that it is \W{2}-complete for $R \in \SBs \SBs L \SEs, \SBs L,D \SEs \SEs$,
that it is \W{\ensuremath{t}}-hard for all $t \geq 1$, for $R = \SBs V \SEs$,
and that it is \paraPSPACE{}-complete for $R = \SBs D \SEs$.
These results are summarized in Table~\ref{table:results}.

\begin{table}[tb]
\newcommand{\wonec}{$\W{1}$-complete}
\newcommand{\wh}{$\W{2}$-hard}
\newcommand{\wc}{$\W{2}$-complete}
\newcommand{\wpc}{$\W{P}$-complete}
\newcommand{\wth}{$\W{\ensuremath{t}}$-hard for all $t \geq 1$}
\newcommand{\paranph}{$\paraNP{}$-hard}
\newcommand{\paranpc}{$\paraNP{}$-complete}
\newcommand{\parapspacec}{$\paraPSPACE{}$-complete}
  \centering
\small{
  \begin{tabular}{@{}l@{\qquad}l@{}} \toprule
    $R $ & $\rcasemod{}$ \\ \midrule
    $\SBs \Asf \SEs$ & \FPT{} \hfill(Thm~\ref{thm:a-casemod}) \\
    $\SBs \Vsf, \Dsf \SEs$ & \FPT{} \hfill(Cor~\ref{cor:vd-casemod}) \\
    $\SBs \Lsf, \Vsf \SEs$ & \wonec{} \hfill(Thm~\ref{thm:lv-casemod}) \\
    $\SBs \Lsf \SEs$ & \wc{} \hfill(Prop~\ref{prop:l-casemod}) \\
    $\SBs \Lsf, \Dsf \SEs$ & \wc{} \hfill(Cor~\ref{cor:ld-casemod}) \\
    $\SBs \Vsf \SEs$ & \wth{}\qquad \hfill(Thm~\ref{thm:v-casemod}) \\
    $\SBs \Dsf \SEs$ & \parapspacec{} \hfill(Thm~\ref{thm:d-casemod}) \\
    \bottomrule
  \end{tabular}
}
  
  \caption{Map of parameterized complexity results.}
  \label{table:results}
  
\end{table}
The modification of a plan (or a case) concerns addition of actions to the plan stored in the case.
One intuitive way to restrict the amount of modification
that is allowed in order to reuse the case
is to restrict the number of allowed additional steps,
resulting in the $\Lsf$ restriction.
It is believed (\citealt{KamHen92}) that the presence of a similar solution, or rather the fact that only $k$ actions need to be added to the stored plan $c$ in order to find the plan $p$, will make the decision problem of existence of $p$ (and also its generation) easier than if no suitable solution $c$ is available. Unfortunately, the following result shows that the corresponding problem $\scasemod{\Lsf}$ remains hard. 
\begin{proposition}
\label{prop:l-casemod}
$\scasemod{\Lsf}$ is \W{2}-complete.
\end{proposition}
\begin{proof}
The result follows from the \W{2}-completeness proof
of finding a solution plan of at most $k$ action occurrences (the $k$-step planning problem)
given by %\citeauthor{BackstroemJonssonChenOrdyniakSzeider12} 
\shortcite{BackstroemJonssonChenOrdyniakSzeider12}.
They proved that \W{2}-hardness already holds for complete goal states.
Now, by letting $(c,J) = (\epsilon,G)$,
the $k$-step planning problem directly reduces
to $\scasemod{\Lsf}$.

To show \W{2}-membership, we sketch the following reduction
to the $k$-step planning problem.
We introduce an additional operator
$(J \cup \SBs \star \mapsto 0 \SEs \Rightarrow J[c] \cup \SBs \star \mapsto 1 \SEs)$,
where $\star$ is a fresh variable.
Furthermore, we let $\SBs \star \mapsto 0 \SEs \in I$ and $\SBs \star \mapsto 1 \SEs \in G$,
and we let $k' = k+1$.
It is straightforward to verify that this reduces $\scasemod{\Lsf}$
to the $k'$-step planning problem.
\end{proof}

\noindent Intuitively, the reason of such a result is that the large number of different actions to choose from
is a source of hardness.
%The proof of this fixed-parameter intractability result is not surprising since the reduction to the problem of finding a solution plan of at most $k$ steps is quite straightforward. We can use the close resemblance of the two problems to improve the fixed-parameter intractability. 
\citeauthor{BackstroemJonssonChenOrdyniakSzeider12} showed that for the $k$-step planning problem, complexity results can be improved by considering only planning instances whose actions satisfy the condition of post-uniqueness. 
Similarly, we can require the set of actions $A'$,
from which glue steps can be taken, 
to be post-unique (\citealt{BackstroemJonssonChenOrdyniakSzeider12}).
This parameterized problem is in fact fixed-parameter tractable
(this result follows from Theorem~5 in \cite{BackstroemJonssonChenOrdyniakSzeider12}).
%In fact, adding this requirement of
%selecting the glue steps from a post-unique
%set of actions allows us to identify actions relevant for reaching the goal and reduces the cardinality of $A'$,
%yields fixed-parameter tractability.

%\begin{proposition}
%\label{prop:l-casemod-pu}
%$\scasemod{\Lsf}\postunique$ is in \FPT.
%\end{proposition}
%\begin{proof}[sketch]
%The result can straightforwardly be obtained by using the algorithm given
%by \citeauthor{BackstroemJonssonChenOrdyniakSzeider12} (2012, Theorem~5)
%to prove that finding a solution plan of length at most $k$
%for instances satisfying post-uniqueness
%is fixed-parameter tractable.
%%%
%In order to show this result,
%we will use the  algorithm given
%by \citeauthor{BackstroemJonssonChenOrdyniakSzeider12} (2012, Theorem~5)
%to prove that finding a solution plan of length at most $k$
%is fixed-parameter tractable, when the set of actions
%satisfies the condition of post-uniqueness.
%This algorithm can be used to find the minimal plan length
%needed to reach a given (partial) goal state from the initial state.
%%
%In order to decide $\scasemod{\Lsf}\postunique$,
%we can use this algorithm to find the length of the minimal plan
%needed to reach state $J$ from the initial state $I$,
%and the length of the minimal plan needed to reach
%the goal $G$ from the state $J[c]$, resulting from state $J$
%by applying plan $c$.
%Checking if the sum of these minimal lengths is at most $k$
%then suffices to decide the original problem.
%\end{proof}

In a similar way, parameterizing directly by the cardinality of $A'$ also provides fixed-parameter tractability:

\begin{theorem}
\label{thm:a-casemod}
$\scasemod{\Asf}$ is in \FPT.
\end{theorem}
\begin{proof}
We have that $k = \Card{A'}$.
Then the number of states $s'$ reachable from any state $s$
by actions from $A'$ is bounded by a function of $k$
and can be enumerated in fixed-parameter
tractable time.
Similar bounds hold for all the states $s''$ reachable
from $s'[c]$ for each such $s'$.
Overall, checking if any of these states $s''$ satisfies the goal state
can thus be done in fixed-parameter tractable time.
\end{proof}

\noindent As a consequence of Theorem~\ref{thm:a-casemod},
we get another fixed-parameter tractability result.

\begin{corollary}
\label{cor:vd-casemod}
$\scasemod{\Vsf,\Dsf}$ is in \FPT.
\end{corollary}
\begin{proof}
If the set $A'$ of actions refers to at most $k$ variables
and at most $m$ values,
then the number of different actions that $A'$ can possibly
contain is bounded by $(m+1)^{2k}$.
The result then follows from Theorem~\ref{thm:a-casemod}.
\end{proof}

\noindent The fact that the above results are the only fixed-parameter tractable results
under the considered restrictions suggests that plan reuse is not the answer to the high computational complexity of planning in general. However, we can use these results to identify settings in which plan reuse is likely to perform well.
%Theorem~\ref{thm:a-casemod} implies that for example replanning in case of an execution failure in situations where the number of applicable actions is restricted by, e.g., limited resources is (fixed-parameter)-tractable if implemented as a plan reuse.
%For example, Theorem~\ref{thm:a-casemod} suggests that replanning in case of an execution failure is tractable to implement as plan reuse
%in situations where the number of applicable actions is restricted by limited resources.
For example, Theorem~\ref{thm:a-casemod} suggests that replanning in case of an execution failure is tractable to implement as plan reuse, provided that the number of applicable actions is limited due to, e.g., limited resources.

These results for plan reuse as implemented in case-based planning are quite unpleasant as in such settings usually $A = A'$. Additionally, $|A|$ tends to be quite high, as the set of actions is obtained by grounding a set of (few) operators (propositional implication rules) over a set of potentially many objects, giving a rise to a rich set of actions which only very rarely satisfies the condition of post-uniqueness to make $\scasemod{\Lsf}$ fixed-parameter tractable. Nevertheless, these claims suggest that, besides identifying \emph{where} to apply the glue steps, a case-based planning system needs to employ heuristics to identify \emph{which} glue steps may be useful.
Even though $\scasemod{\Vsf,\Dsf}$ is in \FPT, parameterizing only on
the number of variables occurring in actions in $A'$,
or only on the number of values occurring in actions in $A'$,
yields fixed-parameter intractability.

\begin{theorem}
\label{thm:v-casemod}
$\scasemod{\Vsf}$ is \W{\ensuremath{t}}-hard for all $t \geq 1$.
\end{theorem}
\begin{proof}
We prove the result by giving an fpt-reduction from $\LCSI{}$,
which is \W{\ensuremath{t}}-hard for all $t \geq 1$.
Let the strings $X_1,\dotsc,X_k$ over the alphabet $\Sigma$
and the integer $m$ constitute an instance of $\LCSI{}$.
For a string $X$ of length $l$
we write $X[0]\!\dotsc{}X[l-1]$.
For each $X_i$ we let $l_i = \Card{X_i}$.
We construct an instance of $\scasemod{\Vsf}$
specified by $\Pi = (V,I,G,A)$, $(c,J)$, $A'$ and $M$.
We let $(c,J) = (\epsilon,G)$, $A' = A$
and $M$ be a sufficiently large number
(that is, $M \geq \sum_{1 \leq i \leq k}\Card{X_i} + (k+1)m + k$).
Also, we define:
\[ \begin{array}{r l}
  V = & \SBs v_1,\dotsc,v_k,s_1,\dotsc,s_k,t_1,\dotsc,t_k,w \SEs; \\
 D(v_i) = & \SBs 0,\dotsc,l_i \SEs; \\
 D(s_i) = & \Sigma \cup \SBs \star \SEs; \\
 D(t_i) =  & \SBs \mtext{none},\mtext{read},\mtext{used} \SEs; \\
 D(w) = & \SBs 0,\dotsc,m \SEs; \\
 A = & A_{\mtext{skip}} \cup A_{\mtext{read}} \cup A_{\mtext{check}} \cup A_{\mtext{finish}}; \\
 A_{\mtext{skip}} = & \SB (\SBs v_i \mapsto u, t_i \mapsto \mtext{none} \SEs \Rightarrow \SBs v_i \mapsto u+1, \\
      &  t_i \mapsto \mtext{none} \SEs), (\SBs v_i \mapsto u, t_i \mapsto \mtext{used} \SEs \Rightarrow \\
      & \SBs v_i \mapsto u+1, t_i \mapsto \mtext{none} \SEs) \SM 1 \leq i \leq k, \\
      & 0 \leq u < l_i \SE; \\
  A_{\mtext{read}} = & \SB (\SBs v_i \mapsto u, t_i \mapsto \mtext{none} \SEs \Rightarrow \SBs s_i \mapsto X_i[u], \\
      & t_i \mapsto \mtext{read} \SEs) \SM 1 \leq i \leq k, 0 \leq u < l_i \SE; \\
  A_{\mtext{check}} = & \SB (\SBs t_1 \mapsto \mtext{read}, \dotsc, t_k \mapsto \mtext{read}, s_1 \mapsto \sigma, \dotsc, \\
      & s_k \mapsto \sigma, w = u \SEs \Rightarrow \SBs t_1 \mapsto \mtext{used}, \dotsc, \\
      & t_k \mapsto \mtext{used}, w \mapsto u+1 \SEs) \SM 0 \leq u < m, \\
      & \sigma \in \Sigma \SE; \\
  A_{\mtext{finish}} = & \SB (\SBs v_i \mapsto l_i \SEs \Rightarrow \SBs t_i \mapsto \mtext{none}, s_i \mapsto \star \SEs) \SM \\
      & 1 \leq i \leq k \SE; \\
  I = & \SB v_i \mapsto 0, s_i \mapsto \star, t_i \mapsto \mtext{none} \SM 1 \leq i \leq k \SE\ \cup \\
      & \SBs w \mapsto 0 \SEs; \mtext{ and} \\
  G = & \SB v_i \mapsto l_i, s_i \mapsto \star, t_i \mapsto \mtext{none} \SM 1 \leq i \leq k \SE\ \cup \\
      & \SBs w \mapsto m \SEs. \\
 \end{array} \]
Note that $\Card{V} = 3k+1$.

The idea behind the reduction is that any solution plan
that results in an assignment of variable $w$ to any $d \geq 1$
corresponds to a witness that the strings $X_1,\dotsc,X_k$
have a common subsequence of length $d$.
The variables $v_1,\dotsc,v_k$ correspond to the position of reading heads
on the strings that can only move from left to right,
and the variables $s_1,\dotsc,s_k$ are used to
read symbols in the string on the position of the reading heads.
The variables $t_1,\dotsc,t_k$ are used to ensure that each symbol
is read at most once (each symbol is either read by using an action in $A_{\mtext{read}}$
or skipped by using an action in $A_{\mtext{skip}}$).
Then the variable $w$ can only be increased if in all strings
the same symbol is read
(by using an action in $A_{\mtext{check}}$).
The actions $A_{\mtext{finish}}$ are used to be able to enforce a complete goal state.

It is now straightforward to verify that there exists a common subsequence
$X$ for $X_1,\dotsc,X_k$ of length $m$ if and only if
the constructed instance is a yes-instance. % for $\scasemod{\Vsf}$.
\end{proof}

\noindent As mentioned above in the preliminaries,
if we restrict the planning instances to Boolean values,
we get a framework corresponding to the STRIPS planning framework.
%Even though $\scasemod{\Vsf}$ in general is \W{\ensuremath{t}}-hard for all $t \geq 1$,
By the fact that $\scasemod{\Vsf,\Dsf}$ is fixed-parameter tractable,
we get that $\scasemod{\Vsf}$ for STRIPS instances
is also fixed-parameter tractable.

By naively keeping track
of all states reachable from the initial state (which are at most $n^k$ many,
for $n = \Card{D}$ and $k = \Card{V}$)
we get that $\scasemod{\Vsf}$ can be solved in polynomial time
for each constant value of $k$.
As a consequence, the following theorem shows
that $\scasemod{\Dsf}$ is of higher complexity
than $\scasemod{\Vsf}$
(unless $\P{}=\PSPACE{}$).

\begin{theorem}
\label{thm:d-casemod}
$\scasemod{\Dsf}$ is \paraPSPACE{}-complete.
\end{theorem}
\begin{proof}
The \paraPSPACE{}-membership result follows from the fact that
$\scasemod{\Dsf}$, when unparameterized,
is in \PSPACE{} (\citealt{BackstromNebel96}).

For the hardness result,
consider the case where the number $k$ of values allowed in the
set of actions $A'$ is $2$.
The problem then reduces to the problem of finding a solution plan
for the Boolean planning instance $\Pi$,
in case we let $(c,J) = (\epsilon,G)$.
%\citeauthor{BackstromNebel96}
\shortcite{BackstromNebel96} showed
that finding a solution plan for Boolean planning instances
(even for complete goal states) is \PSPACE{}-hard.
Since this hardness result holds already for a single value of $k$,
the \paraPSPACE{}-hardness result follows (\citealt{FlumGrohe03}).
\end{proof}

\noindent Parameterizing on the combination of the number of allowed additional
steps together with either the number of variables or the number of values
occurring in actions in $A'$ is not enough to ensure fixed-parameter tractability.

\begin{theorem}
\label{thm:lv-casemod}
$\scasemod{\Lsf,\Vsf}$ is \W{1}-complete.
\end{theorem}
\begin{proof}
\W{1}-membership can be proven analogously
to the \W{1}-membership proof
given by \citeauthor{BackstroemJonssonChenOrdyniakSzeider12}
(\citeyear{BackstroemJonssonChenOrdyniakSzeider12}, Theorem~4)
for the $k$-step planning problem
restricted to actions with one postcondition.
In this proof the problem is reduced
to a certain first-order model checking problem.

%\enlargethispage*{3mm}

For the hardness result, we reduce from the \W{1}-complete problem $\partitionedclique{}$.
Let $(V,E,k)$ be an instance of $\partitionedclique{}$,
where $V$ is partitioned into $V_1,\dotsc,V_k$.
We define the instance $(\Pi,(\Pi',c),A',k')$ of $\scasemod{\Lsf,\Vsf}$
as follows: $\Pi = (W,I,G,A)$,
$\Pi'$ is specified by its initial state $J$,
$(c,J) = (\epsilon,G)$, $k' = 2k + \binom{k}{2}$ and $A' = A$.
We define:
\[ \begin{array}{r l}
  W = & \SBs x_1,\dotsc,x_k \SEs \cup \SB y_{i,j} \SM 1 \leq i < j \leq k \SE; \\
  D(x_i) = & V_i \cup \SBs \star \SEs \mtext{ for all $x_i$ (and arbitrary $\star \not\in V$);} \\
  D(y_{i,j}) = & \SBs \false,\true \SEs \mtext{ for all $y_{i,j}$;} \\
  A = & \SB \mtext{guess}^i_d, \mtext{clear}^i_d \SM 1 \leq i \leq k, d \in V_i \SE\ \cup \\[2pt]
      & \SB \mtext{check}^{v,w}_{i,j} \SM 1 \leq i < j \leq k, v \in V_i, \\[2pt]
      & w \in V_j, \SBs v, w \SEs \in E \SE; \\
  \mtext{guess}^i_d = & (\emptyset \Rightarrow \SBs x_i \mapsto d \SEs), \mtext{ for each $\mtext{guess}^i_d$;} \\
  \mtext{clear}^i_d = & (\emptyset \Rightarrow \SBs x_i \mapsto \star \SEs), \mtext{ for each $\mtext{clear}^i_d$;} \\[2pt]
  \mtext{check}^{v,w}_{i,j} = & (\SBs x_i \mapsto v, x_j \mapsto w \SEs \Rightarrow \SBs y_{i,j} \mapsto \true \SEs), \\
      & \mtext{ for each $\mtext{check}^{v,w}_{i,j}$\!;} \\[2pt]
%\end{array} \]
%\[ \begin{array}{r l}
  I = & \SB x_i \mapsto \star \SM 1 \leq i \leq k \SE\ \cup \\
      & \SB y_{i,j} \mapsto \false \SM 1 \leq i < j \leq k \SE; \mtext{ and} \\
  G = & \SB x_i \mapsto \star \SM 1 \leq i \leq k \SE\ \cup \\
      & \SB y_{i,j} \mapsto \true \SM 1 \leq i < j \leq k \SE. \\
\end{array} \]

The intuition behind the reduction is as follows.
The budget of $k'$ actions allows for $k$ guessing steps,
to set the variables $x_i$ using actions $\mtext{guess}^i_d$;
$\smash{\binom{k}{2}}$ verification steps, to set the variables $y_{i,j}$ to $1$
using actions $\mtext{check}^{v,w}_{i,j}$;
and $k$ cleanup steps,
to reset the variables $x_i$ using actions $\smash{\mtext{clear}^i_d}$.
The only way to achieve the goal state is by guessing
a $k$-clique.

It is now straightforward to verify that the graph $(V,E)$ has a $k$-clique
if and only if there exist plans $p,p'$ of total length $k'$
such that $I[p] = J$ and $J[c][p']$ satisfies $G$.
\end{proof}

\begin{corollary}
\label{cor:ld-casemod}
$\scasemod{\Lsf,\Dsf}$ is \W{2}-complete.
\end{corollary}
\begin{proof}
The claim follows directly from
the proof of Proposition~\ref{prop:l-casemod},
since $k$-step planning is \W{2}-complete
already for Boolean planning instances
(\citealt{BackstroemJonssonChenOrdyniakSzeider12}).
\end{proof}

\noindent The above results together give us
the complete parameterized complexity characterization
as depicted in Figure~\ref{fig:lattice}.

\paragraph{Reusing an infix of the case}
As a slight generalization of the \casemod{} problem,
we consider the problem \casemodstar{}.
In this problem, we require not that the full plan $c$ from the case is being reused together with its initial state $J$, but that any infix $c'$ of the plan
(i.e., any subplan $c'$ resulting from removing any prefix and postfix from $c$) is reused with its corresponding initial state $J'$. Formally, the question becomes whether there exists a sequence of actions $(g_1,\dotsc,g_m) \in (A')^m$ for some $m \leq M$, and whether there exists some $0 \leq i \leq m$ and some $1 \leq i_1 \leq i_2 \leq l$ such that $(g_1,\dotsc,g_i,c_{i_1},\dotsc,c_{i_2},g_{i+1},\dotsc,g_m)$ solves the new planning instance $\Pi$ and $I[(g_1,\dotsc,g_i)] = J[(c_1,\dotsc,c_{i_1-1})]$, where $c = (c_1,\dotsc,c_l)$.

The following results
show that this generalization does not change the parameterized
complexity results that we obtained in the previous section.

\begin{observation}
Whenever $\rcasemod{}$ is in \FPT,
then also $\rcasemodstar{}$ is in \FPT.
\end{observation}
\begin{proof}
Let $c = (c_1,\dotsc,c_n) \in A^n$.
There are only $n^2$ different ways of selecting
subplans $(c_d,\dotsc,c_e)$ to consider,
for $1 \leq d \leq e \leq n$.
For each of these, we can compute the
initial state $J[(c_1,\dotsc,c_{d-1}]$ in linear time.
Simply trying all these $n^2$ possibilities
using the algorithm for $\rcasemod{}$
results in an fixed-parameter tractable algorithm for $\rcasemodstar{}$.
\end{proof}

\begin{theorem}
The completeness results in Proposition~\ref{prop:l-casemod}, Theorems~\ref{thm:v-casemod},~\ref{thm:d-casemod}~and~\ref{thm:lv-casemod} and Corollary~\ref{cor:ld-casemod}
also hold for the corresponding variants for $\rcasemodstar{}$.
\end{theorem}
\begin{proof}[sketch]
For the hardness results, it suffices to note that the hardness
proofs of these theorems
use a case containing the empty plan $\epsilon$.

For the membership results,
we note that the $\rcasemodstar{}$ problem
can be solved by answering the disjunction of
polynomially many (independent) $\rcasemod{}$ instances.
The \W{\ensuremath{t}}-membership results can then be proved
by encoding the $\rcasemod{}$ instances as
instances of certain first-order model checking problems
(\citealt{BackstroemJonssonChenOrdyniakSzeider12}),
and combining these into
one model checking problem instance
that is equivalent to the disjunction of the
separate $\rcasemod{}$ instances.
%
%To prove the \W{\ensuremath{t}} membership results, it then suffices to note that we can reduce all such $\rcasemod{}$ instances
%to suitable first-order model checking problems,
%which can be combined into a suitable single first-order model checking problem
%by taking the disjunction of all formulas and the union of all models involved.
%
For the \paraPSPACE{}-membership result,
we can straightforwardly evaluate the disjunction of the $\rcasemod{}$ instances
in polynomial space.
\end{proof}

\paragraph{Generalized infix reuse}
The problem of $\casemod{}$ can be generalized even further by reusing an infix of the stored solution plan $c$ from \emph{any state} that satisfies the preconditions of the plan infix.
%, denoted $\mbox{\textsc{Case}}^*\mbox{\textsc{Mod}}$.
For this problem, the instances coincide with those of $\casemod{}$, but the question is:

\begin{quote}
  \emph{Question:} Does there exist a sequence of actions $(g_1,\dotsc,g_m) \in (A')^m$
  for some $m \leq M$, and does there exist $0 \leq i \leq m$ and  $0~\leq~i_1~\leq~i_2~\leq~l$
  such that $(g_1,\dotsc,g_i,c_{i_1},\dotsc,c_{i_2},g_{i+1},\dotsc,g_m)$
  is a solution plan for $\Pi$ and
  for all $i_1 \leq j \leq i_2$ we have that action $c_j$ is applicable in $I[(g_1,\dotsc,g_i,c_{i_1},\dotsc,c_{j-1})]$?
  
\end{quote}

\noindent This generalization does not change the parameterized complexity results stated in Table~\ref{table:results}.
In all cases where $\rcasemod{}$ is fixed-parameter tractable, we can
obtain a fixed-parameter tractable algorithm to solve the above problem,
since in those cases we can enumerate all states reachable from a given state
in fixed-parameter tractable time.
For the fixed-parameter intractability results,
the hardness follows straightforwardly
from the hardness proofs
for the corresponding $\rcasemod{}$ problems. 

\paragraph{Reusing a sequence of actions}
In principle, there is no need to require anything from the state to which the stored plan is applied.
%that the initial state $J$ given in the case is being reused.
Therefore we will consider the following generalization
of the $\casemod{}$ problems discussed above.
In this problem, denoted by $\planmod$,
we remove the requirement that
the additional steps added before the plan $c$
result in the initial state $J$ or some other state that satisfies the preconditions of (the infix of) the plan $c$.
Also, we allow the insertion of additional steps
in the middle of the plan $c$.
Formally, the question then becomes whether
there exists some $m \leq k$,
a sequence of actions $g = (g_1,\dotsc,g_m) \in A^{m}$,
and some sequence of actions $p = (p_1,\dotsc,p_{l+m})$
such that $p$ is a solution plan of $\Pi$
and $p$ can be divided into two subsequences
$c$ and $g$, i.e.,
interleaving $c$ and $g$ yields $p$.
In other words, the additional steps $g$
can be used anywhere before, after or in the middle of the plan $c$.
Since we do not restrict ourselves
to any particular state being visited in our solution plan,
the actions in the glue sequence $g$
can be used anywhere before, after or in the middle of the plan $c$.
Similarly to the case for $\casemod$,
we define the variants $\planmodstar{}$, $\rplanmod{}$
and $\rplanmodstar{}$.
% % MOD %
In the following, we show that all
variants of $\planmod$
are fixed-parameter intractable.

\begin{theorem}
$\splanmod{\Lsf,\Vsf,\Dsf,\Asf}$ is \W{P}-complete.
\end{theorem}
\begin{proof}[sketch]
For \W{P}-membership, we sketch how to reduce $\splanmod{\Lsf,\Vsf,\Dsf,\Asf}$
to the problem of determining whether a nondeterministic Turing machine $T$
accepts the empty string within a bounded number of steps
using at most $k'$ nondeterministic steps
(parameterized by $k'$).
Since this parameterized halting problem is in \W{P}
(\citealt{Cesati03}),
this suffices to show \W{P}-membership.
First $T$ guesses $k$ pairs $(m_i,a_i)$,
for $0 \leq m_i \leq \Card{c}$ and $a_i \in A'$.
Pair $(m_1,a_1)$ corresponds to the application of the first $m_1$
actions from the given plan $c$, followed by the application of action $a_i$.
Similarly, for each $i > 1$, pair $(m_i,a_i)$ corresponds
to the application of the next $m_i$ actions
from the given plan $c$, followed by the application of action $a_i$.
Then $T$ (deterministically) verifies whether applying the plan corresponding
to $(m_1,a_i),\dotsc,(m_k,a_k)$ is a solution plan.

To prove \W{P}-hardness, we reduce from \pwsatcirc{}.
Let $C$ be a circuit for which we want to
check whether there exists a satisfying
assignment of weight at most $k$.
Let $x_1,\dotsc,x_n$ be the input nodes,
$y_1,\dotsc,y_m$ the internal nodes,
and $z$ the output node of $C$,
together denoted $\mtext{nodes}(C)$.
We assume without loss of generality
that $C$ contains only {\sc and} and {\sc negation} nodes.
Since $C$ is acyclic, we let the sequence $(g_1,\dotsc,g_l)$
denote the nodes of $C$ in any order such that
for each $g_i$ we have that $j < i$ for all input nodes $g_j$ of $g_i$.
We construct an instance of $\splanmod{\Lsf,\Vsf,\Dsf,\Asf}$
consisting of a planning instance $\Pi = (V,I,G,A)$,
a plan $c$, and an integer $k'$.
We let $k' = k$, and we define:
\[ \begin{array}{r l}
  V = & \mtext{nodes}(C) \cup \SBs \sigma \SEs; \\
  D(v) = & \SBs \false, \true \SEs \mtext{ for all $v \in V$}; \\
  I = & \SB v \mapsto 0 \SM v \in V \SE; \\
  G = & \SBs z \mapsto \true \SEs; \\
  A = & \SBs a^{x_1}, \dotsc, a^{x_n} \SEs \cup \SBs a^{\mtext{on}}, a^{\mtext{off}} \SEs \cup \SBs a^{g_1}, \dotsc, a^{g_l} \SEs; \\
  A' = & \SBs a^{\mtext{on}} \SEs; \\
  a^{x_i} = & (\SBs \sigma \mapsto \true \SEs \Rightarrow \SBs x_i \mapsto \true \SEs) \mtext{ for each $x_i$}; \\
  a^{\mtext{on}} = & (\emptyset \Rightarrow \SBs \sigma \mapsto \true \SEs); \\
  a^{\mtext{off}} = & (\emptyset \Rightarrow \SBs \sigma \mapsto \false \SEs); \mtext{ and} \\
  c = & (a^{\mtext{off}},a^{x_1},a^{\mtext{off}},a^{x_2},\dotsc,a^{\mtext{off}},a^{x_n},a^{g_1},\dotsc,a^{g_l}). \\
\end{array} \]
We define $a^{g_i}$ for each $g_i$ as follows.
If $g_i$ is a {\sc negation} node with input $y$,
we let $a^{g_i} = (\SBs y \mapsto \false \SEs \Rightarrow \SBs g_i \mapsto \true \SEs)$.
If $g_i$ is an {\sc and} node with inputs $y_1,\dotsc,y_u$,
we let $a^{g_i} = (\SBs y_1 \mapsto \true, \dotsc, y_u \mapsto \true \SEs \Rightarrow \SBs g_i \mapsto \true \SEs)$.

It is now straightforward to verify that $(\Pi,c,k')$ is
a yes-instance of $\splanmod{\Lsf,\Vsf,\Dsf,\Asf}$ if and only if
the circuit $C$ has a satisfying assignment of weight $k$.
\end{proof}

\noindent Note that the proof of the above theorem
suffices to show fixed-parameter intractability
of all variants of the $\planmod$
and $\planmodstar$ problems.
We also point out that this fixed-parameter intractability
result holds even when the problem is restricted to instances
for which the entire set $A$ of actions
satisfies post-uniqueness.

%\begin{corollary}
%$\splanmod{\Lsf,\Vsf,\Dsf,\Asf}$,
%restricted to instances for which the entire set of actions satisfies the condition
%of post-uniqueness, is \W{P}-complete.
%\end{corollary}

%\enlargethispage*{2mm}

\section{Conclusion}
  
  We provided theoretical results, using the framework of parameterized complexity,
  to identify situations in which plan reuse is provably tractable.
  We drew a detailed map of the parameterized complexity landscape
  of several variants of problems that arise in the context of
  case-based planning. In particular, we considered the problem of
  reusing an existing plan, imposing various restrictions in terms of
  parameters, such as the number of steps that can be added to the
  existing plan to turn it into a solution of the planning instance at
  hand.

  The results show that contrary to the common belief, the fact that
  the number of modifying actions is small
  does not guarantee tractability on its own. We additionally need to
  restrict the set of actions that can
  participate in the modifications.
  This indicates the need for a good heuristic function that identifies
  a limited set of actions used for modifications.

  In the future, these results may be extended to richer planning
  formalisms, e.g., considering variables of
  different types or using predicates to express certain properties
  related to a planning domain rather than
  planning instance.

\cleardoublepage\pagebreak
\bibliographystyle{aaai}

\end{document}